\newif\iftcs
\tcsfalse
\tcstrue %

\documentclass[11pt]{article}
\usepackage{amssymb}
\usepackage{amsmath,amsthm}
\usepackage{url}
\usepackage{color}
\iftcs
  \usepackage{fullpage}
\fi

\usepackage{ifpdf}
\ifpdf    %
\usepackage{hyperref}
\else    %
\usepackage[hypertex]{hyperref}
\fi

\IfFileExists{mysymbols.sty}{
  \usepackage{mysymbols}
  \usepackage{mytheo}
}{
  \usepackage{../../latex_files/mysymbols}
  \usepackage{../../latex_files/mytheo}
}

\newcommand{\knote}[1]{\textcolor{red}{$\ll$\textsf{#1 --Aryeh}$\gg$\marginpar{\tiny\bf AK}}}

\usepackage{color}

\newcommand{\X}{\mathcal{X}}
\newcommand{\htil}{{\tilde h}}
\newcommand{\hemp}{{h^*}}

\newcommand{\Ltil}{{\tilde L}}
\newcommand{\Lbar}{{\bar L}}
\newcommand{\etbar}{{\bar\eta}}
\newcommand{\ettil}{{\tilde\eta}}
\newcommand{\etstr}{{\eta^*}}
\renewcommand{\hbar}{{\bar h}}

\newcommand{\G}{\mathcal{G}}
\newcommand{\F}{\mathcal{F}}
\renewcommand{\H}{\mathcal{H}}
\def\eps{\varepsilon}

\newcommand{\E}{\mathbb{E}}
\newcommand{\M}{\mathbb{M}}
\renewcommand{\P}{\mathbb{P}}

\newcommand{\ddim}{\mathrm{ddim}}
\renewcommand{\diam}{\mathrm{diam}}

\usepackage{bbm}

\newcommand{\tsabs}[1]{| #1 |}
\newcommand{\wh}{\widehat}

\newcommand{\Lip}[1]{\nrm{#1}_{\textrm{{\tiny \textup{Lip}}}}}

\newcommand{\Liprho}[2]{\nrm{#1}_{\textrm{{\tiny \textup{Lip(\ensuremath{#2})}}}}}
\newcommand{\tsLip}[1]{\tsnrm{#1}_{\textrm{{\tiny \textup{Lip}}}}}
\newcommand{\tsnrm}[1]{\| #1 \|}

\newcommand{\citet}[1]{\cite{#1}}
\newcommand{\citep}[1]{\cite{#1}}

\newcommand{\rad}{\wh{\mathcal{R}}}

\newcommand{\risk}{R}
\newcommand{\distr}{\mu}

\newcommand{\turb}[1]{\left\llbracket #1 \right\rrbracket}
\newcommand{\qloss}[1]{q\circ{#1}}

\newcommand{\logtwo}{\log_2}
\newcommand{\lognat}{\ln}

\iftcs
\newtheorem{theorem}{Theorem}[section]
\else
\newtheorem{theorem}{Theorem}
\fi
\newtheorem{lemma}[theorem]{Lemma}
\newtheorem{corollary}[theorem]{Corollary}

\renewenvironment{proof}{{\bf Proof:\ }}{\hfill$\Box$\medskip}

\begin{document}

\title{Efficient Regression in Metric Spaces\\ via Approximate Lipschitz Extension
\footnote{A preliminary version of this paper appeared 
in Proceedings of SIMBAD 2013 conference \cite{DBLP:conf/simbad/GottliebKK13}.}
}

\author{
Lee-Ad Gottlieb\thanks{Ariel University, \texttt{leead@ariel.ac.il}.
}
and
Aryeh Kontorovich\thanks{Ben-Gurion University, \texttt{karyeh@cs.bgu.ac.il}.
This research was partially supported by the Israel Science Foundation 
(grant \#1141/12)
and a Yahoo Faculty award.
}
and
Robert Krauthgamer\thanks{Weizmann Institute of Science, 
\texttt{robert.krauthgamer@weizmann.ac.il}. 
This work was supported in part by the Israel Science Foundation
(grants \#452/08 and \#897/13), and by a Minerva grant.}
}

\maketitle

\begin{abstract}
We present a framework for performing efficient regression in general 
metric spaces. 
Roughly speaking, our regressor predicts the value at a new point
by computing an approximate Lipschitz extension ---
the smoothest function consistent with the observed data --- 
after
performing 
structural risk minimization
to avoid overfitting.
We obtain
finite-sample risk bounds
with
minimal structural and noise assumptions,
and 
a natural runtime-precision tradeoff.
The offline (learning) and online (prediction) stages can be solved 
by convex programming, but this naive approach has runtime complexity $O(n^3)$,
which is prohibitive for large datasets.
We design instead a regression algorithm 
whose speed and generalization performance depend on the intrinsic
dimension of the data, to which the algorithm adapts.
While our main innovation is algorithmic, the statistical results may also be of independent interest.
\end{abstract}

\section{Introduction}
\label{sec:intro}
The classical problem of estimating a continuous-valued function from noisy observations, known as {\em regression},
is of central importance in statistical theory with a broad range of applications, see e.g. 
\cite{MR740865,MR726392,MR1093466,MR1161622,MR1920390}.
When no structural assumptions concerning the target function are made, the regression problem is termed 
{\em nonparametric}.
Informally, the main objective in the study of nonparametric regression is to understand the 
relationship between the regularity conditions that a function class might satisfy 
(e.g., Lipschitz or H\"older continuity, or sparsity in some representation) and 
the minimax risk convergence rates 
\cite{MR2013911,MR2172729}.
A further consideration is the computational efficiency of constructing the regression function.

The general (univariate) nonparametric regression problem may be 
stated as follows.
Let $(\X,\rho)$ 
be a metric space, 
and
let $\H$ be
a collection 
of functions (``hypotheses'') $h:\X\to[0,1]$.
(Although in general, $h$ need not be restricted to a bounded range,
typical assumptions on the diameter of $\X$ and the noise distribution amount to an effective
truncation \cite{DBLP:conf/icml/kon14,MeirZhang2003}.)
The space $\X\times[0,1]$ is endowed with some fixed, unknown probability distribution $\distr$,
and the learner observes $n$ iid draws $(X_i,Y_i)\sim \distr$. The learner then seeks to fit the observed data
with some hypothesis $h\in\H$ so as to minimize the  
{\em risk}, usually defined as the expected loss $\E \abs{h(X)-Y}^q$ 
for $(X,Y)\sim\distr$ and some $q\ge1$. This 
is known 
in machine learning theory
as the {\em agnostic} setting.
The agnostic setting is considerably more general than the additive
(typically Gaussian) noise model prevalent in statistics
(see \cite{DBLP:journals/corr/RakhlinST13} for a recent
paper on agnostic regression).

We consider two kinds of risk: $L_1$ (mean absolute) and 
$L_2$ (mean square). More precisely,
for $q\in\set{1,2}$
we associate
to each hypothesis $h\in\H$
the empirical $L_q$-risk
\beqn
\label{eq:emprisk}
\risk_n(h,q) = 
\oo n\sum_{i=1}^n \abs{h(X_i)-Y_i}^q ,
\eeqn
and the (expected) $L_q$-risk
\beqn
\label{eq:exprisk}
\risk(h,q) = 
\E \abs{h(X)-Y}^q
=\int_{\X\times[0,1]} \abs{h(x)-y}^q d\distr(x,y).
\eeqn

It is well-known that $h(x)=\M[Y\gn X=x]$ (where $\M$ is a median) 
minimizes 
$\risk(\cdot,1)$ 
over all 
integrable
$h\in[0,1]^\X$
and $h(x)=\E[Y\gn X=x]$ minimizes $\risk(\cdot,2)$.
However, these characterizations are of little practical use as neither is computable 
without knowledge of $\distr$.
Hence, the standard route is to minimize the regularized empirical risk
and provide generalization bounds for this procedure.
A naive implementation of this approach requires solving a linear (or quadratic) program,
which incurs a prohibitive $O(n^3)$ time complexity.

\paragraph*{Our contribution}
\label{par:contrib}
Our approach to the regression problem departs from that of classical statistics
in several important ways.
Statistics has traditionally been more concerned with establishing
minimax risk rates
than with the computational efficiency (or even explicit construction) of the regression
procedure.
In contradistinction, our framework involves a precision parameter $\eta$,
which controls the runtime-precision tradeoff. In particular,
this means
that
Bayes-consistency is not achievable
for $\eta>0$.
Further, our results rely on the structure of the metric space, 
but only to the extent of assuming that it
has a low ``intrinsic'' dimensionality.
Specifically, we consider the doubling dimension of $\X$, denoted $\ddim(\X)$,
which was introduced by \cite{DBLP:conf/focs/GuptaKL03}
based on earlier work of \cite{Assouad83,Clarkson99},
and has been since utilized in several algorithmic contexts,
including networking, combinatorial optimization, and similarity search,
see e.g. \cite{KSW09,Talwar04,KL04,BKL06,HM06,CG06,Clarkson06}.
(A formal definition and
typical
examples appear in Section \ref{sec:tech}.)
Following the work of \cite{DBLP:journals/tit/GottliebKK14}
on classification problems, 
our risk bounds and algorithmic runtime bounds are
stated in terms of the doubling dimension 
of the data space
and the Lipschitz
constant of the regression hypothesis, 
although neither of these quantities need
be known in advance.
Note that any continuous function can be uniformly approximated by
Lipschitz functions,
with the Lipschitz constant as a measure of regularity
--- to which our algorithm adapts in a data-dependent fashion. 

Our paper's main contribution is computational. 
The algorithm in Theorem~\ref{thm:risk-minimization} computes
an $\eta$-additive approximation to the Lipschitz-regularized
empirical risk minimizer in time $\eta^{-O(\ddim(\X))}n \lognat^3 n$
(recall $\eta>0$ is a parameter that controls the desired precision). 
By Theorem~\ref{thm:LipExt}, this hypothesis can be evaluated 
on new points in time $\eta^{-O(\ddim(\X))} \lognat n$. 
A novel feature of our construction is the use of a 
spanner to reduce the runtime of a linear program,
and the spanner construction in Appendix~\ref{sec:spanner}
is itself of independent interest, having already been invoked in
\cite{DBLP:conf/stoc/ElkinS13,DBLP:conf/stoc/Solomon14}.
We also present some statistical risk bounds
(culminating in Theorem~\ref{thm:main-risk}).

A simple
no-free-lunch argument shows that it is impossible to learn functions
with arbitrary oscillation, 
and hence
Lipschitzness is a natural and commonly used regularization 
constraint \cite{MR2013911,MR2172729,shwartz2014understanding}.
In this sense, our work fits into the so-called {\em luckiness} paradigm
\cite{DBLP:journals/tit/Shawe-TaylorBWA98},
of which SVM is a classic instance. Rather than guaranteeing a priori Bayes-consistency
or excess risk bounds, luckiness bounds are data-dependent. Thus, in the case of SVM,
a {\em lucky} sample is one that admits a large-margin separator; this in turn allows for
optimistic generalization bounds ---
as opposed to
a less lucky sample with
a smaller margin and correspondingly more pessimistic bounds.
More recently, this data-dependent approach
was applied to general metric spaces
\cite{DBLP:journals/tit/GottliebKK14}
and was later shown to be
Bayes-consistent
\cite{DBLP:conf/aistats/KontorovichW15}.

Our runtime and generalization bounds explicitly depend on the 
doubling dimension of $\X$, 
but as we discuss in Remark~\ref{rem:adapt},
recent results with data-dependent generalization~\cite{GottliebKK13-tcs}
renders our approach adaptive to the {\em intrinsic} dimension of the samples,
offering large savings when the latter is
even moderately smaller
than the 
ambient metric dimension.

\paragraph*{Paper outline} 
We start by defining the basic concepts in Section \ref{sec:tech}. 
Our efficient model selection procedure is described in Section \ref{sec:bv}, 
and the prediction algorithm (for a test point) is described in Section \ref{sec:lipext}. 
The risk guarantees of our method are provided in Section \ref{sec:risk}.

\paragraph*{Related work}
There are many excellent references for classical Euclidean 
nonparametric regression assuming iid noise, 
see for example
\cite{MR1383093,MR1920390}.
For metric regression, a simple risk bound follows from classic VC theory via 
the
pseudo-dimension,
see e.g.~\cite{pollard84,MR1367965,neylon06}.
However, the pseudo-dimension of many
natural
function classes, 
including Lipschitz functions,
is infinite ---
yielding a vacuous
bound.
An approach to nonparametric regression based on empirical risk minimization,
though only for the Euclidean case,
may already be found in \cite{DBLP:journals/tit/LugosiZ95}; see the comprehensive historical overview therein.
Indeed, \cite[Theorem 5.2]{MR1920390} provides a kernel regressor for 
Lipschitz functions that achieves the 
minimax rate. Note however that (a) the setting is restricted to 
Euclidean spaces; and 
(b) the runtime cost of evaluating the hypothesis at a new point grows 
linearly with the sample size 
(while our complexity is roughly logarithmic). 

More recently, risk bounds in terms of doubling dimension and 
Lipschitz constant were given in \cite{NIPS2009_1009}.
These results assumed an additive noise model, 
and hence are incomparable to ours.
Following up, a regression technique based on random 
partition trees was proposed in \cite{Kpotufe2012},
based on mappings between 
Euclidean spaces and also assuming an additive noise model. 
Another recent advance in nonparametric regression was Rodeo \cite{1132.62026}, 
which escapes the curse of dimensionality by adapting to the sparsity
of the regression function. In contrast, our results
apply to general metric spaces and exploit Lipschitz smoothness
rather than sparsity.

Our work was inspired by the paper of von Luxburg and Bousquet 
\cite{DBLP:journals/jmlr/LuxburgB04}, 
who established a connection between 
Lipschitz classifiers in metric spaces and 
large-margin hyperplanes in Banach spaces, thereby providing a 
novel generalization bound for 
nearest-neighbor classifiers. 
They developed a powerful statistical framework whose core idea may be 
summarized as follows: 
to predict the behavior at new points, 
find the smoothest function consistent with the training sample,
and then extend the function to the new points.
Since the regression function is defined implicitly by the labeled sample,
the work of \cite{DBLP:journals/jmlr/LuxburgB04}
raises natural algorithmic issues, such as
efficiently evaluating this function on test points (prediction)
and performing model selection
(Structural Risk Minimization)
to avoid overfitting.
Subsequent work (by the current authors) \cite{DBLP:journals/tit/GottliebKK14} leveraged the doubling dimension 
for both statistical and computational efficiency,
and designed an efficient classifier for doubling metric spaces.
Its key feature is an efficient algorithm to optimize the balance
between the empirical risk and the penalty term for a given input.
The present work extends these techniques from binary classification to real-valued regression, which presents a host of technical challenges.

\section{Technical background}
\label{sec:tech}
We use standard notation and definitions throughout.
The long-standing custom of ignoring measurability issues
in learning-theoretic papers is more than justified in this case:
we (effectively) only consider a class of functions computable
to fixed precision by a fixed algorithm,
and thus no loss of generality is incurred in treating this
set of functions as countable.
We write $\lognat$ for the natural logarithm
and $\log_b$ to specify a different base $b$.

\paragraph*{Metric spaces, Lipschitz constants}
A {\em metric} $\rho$ on a set $\X$ is a symmetric function 
that is positive (except for $\rho(x,x)=0$) 
and satisfies the triangle inequality $\rho(x,y) \le \rho(x,z)+\rho(z,y)$;
together the two comprise the metric space $(\X,\rho)$.
The diameter of a set $A\subseteq\X$
is defined by $\diam(A) = \sup_{x,y\in A} \rho(x, y)$.
There is no loss of generality in assuming $\diam(\X)=1$ 
since we can always scale the distances (when they are bounded).
The {\em Lipschitz constant} of a function $f :\X\to\R$, denoted $\Lip{f}$
(or $\Liprho{f}{\rho}$ if we wish to make the metric explicit)
is defined to be the smallest $L \ge 0$ such that
$|f(x)-f(y)| \le L\rho(x, y)$ 
holds for all $x, y \in\X$. 
In addition to the metric $\rho$ on $\X$,
we will endow the space of all functions $f:\X\to\R$
with
the $L_\infty$ metric:
\beq
\nrm{f-g}_\infty =
\sup_{x\in\X}\abs{f(x)-g(x)}
.
\eeq
A function is called \emph{$L$-Lipschitz} if $\Lip{f}\leq L$.
We will denote by $\H_L$ the collection of all $L$-Lipschitz functions $\X\to[0,1]$.
It will occasionally be convenient to restrict this class to functions
with $\Lip{f}\ge1$; the latter collection will be denoted by $\H_{L\ge1}$.
This incurs no loss of generality in our results, as 
our Structural Risk Minimization procedure in general
selects hypotheses whose Lipschitz constant grows with sample size.
(See for example the risk bound presented at the beginning of Section \ref{sec:bv}.)

\paragraph*{Minkowski sums and perturbations}
If $A,B$ are two families of functions mapping $\X$ to $\R$, then their
{\em Minkowski sum} is $A\oplus B := \set{a+b: a\in A,b\in B}$.
For $\eta>0$, define $\turb{\eta} := {[-\eta,\eta]}^\X$.
Hence, $\H_L\oplus\turb{\eta}$ represents the collection of all 
$[0,1]$-valued $L$-Lipschitz functions 
perturbed pointwise by at most $\eta$.

\paragraph*{Doubling dimension}
For a metric space $(\X,\rho)$, let
$\lambda>0$
be the smallest value such that every
ball in $\X$ can be covered by $\lambda$ balls of half the radius.
The {\em doubling dimension} of $\X$ is $\ddim(\X)=\logtwo \lambda$.
A metric space (or family of metrics) is called {\em doubling}
if its doubling dimension is uniformly bounded. 
Note that while a low Euclidean
dimension implies a low doubling dimension (Euclidean metrics of dimension
$d$ have doubling dimension $O(d)$), low doubling
dimension is strictly more general than low Euclidean dimension.

Doubling metric spaces occur naturally in many data analysis applications,
including for instance the geodesic distance of a low-dimensional manifold 
residing in a possibly high-dimensional space 
assuming mild conditions, e.g., on curvature.
Some concrete examples for doubling metric spaces include:
(i) $\R^d$ for fixed $d$ equipped with an arbitrary
norm,
e.g.\ $\ell_p$ or a mix between $\ell_1$ and $\ell_2$;
(ii) the planar earthmover metric between point sets of fixed size $k$
\cite{DBLP:journals/tit/GottliebKK14};
(iii) the $n$-cycle graph and its continuous version, the quotient $\R/\Z$,
and similarly bounded-dimensional tori.
In addition, various networks that arise in practice, 
such as peer-to-peer communication networks and online social networks,
can be modeled reasonably well by a doubling metric space.

\paragraph*{Graph spanner}
A \emph{$(1+\delta)$-stretch spanner} for a graph $G$ 
(which may have positive edge-lengths) is a subgraph $H$ 
that contains all nodes of $G$ (but not all edges), 
and $\rho_H(u,v) \le (1+\delta) \rho_G(u,v)$ for all $u,v \in G$, 
where $\rho_G(u,v)$ denotes the shortest-path distance 
between $u$ and $v$ in $G$ (and similarly $\rho_H(u,v)$ for $H$). 
If a spanner $H$ achieves this stretch bound even when $\rho_H$ 
is evaluated only on paths in $H$ with at most $k$ edges, 
then $H$ is called a \emph{$(1+\delta)$-stretch $k$-hop spanner} for $G$.

A spanner for a finite metric space $\X$ is defined 
by viewing the metric space as a complete graph $G$ on the vertex set $\X$,
with edge-lengths corresponding to distances in $\X$.
Doubling metrics are known to admit good spanners \cite{CGMZ05,HM06,GR08b}.
We will use a specific variant described in Appendix~\ref{sec:spanner}.

\hide{
\paragraph*{Learning}
We work in the {\em agnostic PAC} learning model \citep{mohri-book2012,shwartz2014understanding}.
The learner receives $n$ labeled examples $(X_i,Y_i)\in\X\times[0,1]$ drawn
iid 
according to some unknown
probability distribution $\P$.
A {\em loss} function $\ell:\R\times\R\to[0,\infty)$ quantifies the performance of
a hypothesis on a labeled example; in this paper, we restrict ourselves to $\ell_q(y,y')=\abs{y-y'}^q$
for $q\in\set{1,2}$.
Associated to any {\em hypothesis} $h:\X\to[0,1]$ is its {\em empirical loss}
\beq
\oo n \sum_{i\in[n]}\ell(h(X_i), Y_i)
\eeq
and {\em generalization error}
\beq
\err(h)=\P(h(X)\neq Y).
\eeq
}

\section{Regression algorithm}
\label{sec:bv}

Let us fix the user-specified 
parameters $q\in\set{1,2}$ (risk type),
$\delta>0$ (confidence level),
and $\eta>0$ (precision parameter).
Given the training sample $(X_i,Y_i)_{i\in[n]}$, our goal is to
construct a hypothesis $\htil:\X\to[0,1]$ with small expected risk $\risk(\htil,q)$.
Since the expected risk cannot be computed exactly (it depends on the unknown
distribution $\mu$), we will instead seek to minimize an upper estimate
of the risk. 
Theorem~\ref{thm:main-risk}
shows that 
with probability at least $1-\delta$,
for all $\Ltil\ge1$,
$\ettil\in\set{\eta,2\eta,\ldots,\eta \lfloor 1/\eta \rfloor,1}$
and hypothesis $\htil \in \H_{\Ltil}\oplus\turb{\ettil}$
(that is, $\htil$ is $\ettil$-close to some $\Ltil$-Lipschitz function),
\beq
\risk(\htil,q) 
&\le&
\risk_n(\htil,q)
+
4(2q-1)\ettil
+
(1+o(1))
\sqrt{\frac{32
\lognat{\frac{8}{
(2q-1)\ettil
}}
}n}
\paren{\frac{16q^{3/2}
\Ltil
}{
(2q-1)\ettil
}}
^{1+\ddim(\X)}
+3\sqrt{\frac{\lognat\frac4{\delta\eta}}{2n}}
.
\eeq
Denote the RHS by $Q(\htil,\Ltil,\ettil)$; when $\Ltil,\ettil$ are clear from the context,
it may be convenient to write just $Q(\htil)$.
In this section, we design an algorithm to find a hypothesis 
that approximately minimizes $Q(\htil,\Ltil,\ettil)$.
(A technique for
quickly evaluating this hypothesis on new points is presented
in Theorem \ref{thm:LipExt}.)

Suppose that for some training sample, 
$Q(\cdot)$ is {\em minimized} by some $(\hemp,L^*,\etstr)$,
where the minimum is taken over $\etstr\ge0$, $L^*\ge1$, and hypothesis $\hemp:\X\to[0,1]$ that is $\etstr$-close to
some $L^*$-Lipschitz function.

\begin{theorem}
\label{thm:risk-minimization}
There is an algorithm that, given a precision parameter $\eta\in(0,\oo4)$ 
and a training sample $(X_i,Y_i)\in \X\times[0,1]$, $i\in[n]$, 
computes $\ettil>0$, $\Ltil\ge1$ and a hypothesis 
$\htil:\X\to[0,1]$, $\htil \in \H_{\tilde{L}}\oplus\turb{ \ettil}$
that satisfy
\beqn \label{eq:risk-minimization}
  Q(\htil,\tilde{L},\ettil) \le Q(\hemp,L^*,\etstr) + \eta ,
\eeqn
in time $\eta^{-O(\ddim(\X))}n \lognat^3 n$.
\end{theorem}

\begin{rem}
The role of the precision parameter $\eta$
is to facilitate the construction of an approximate Lipschitz
hypothesis with much greater efficiency than its exact Lipschitz counterpart.
The bound \eqref{eq:risk-minimization} shows that the computed hypothesis 
$\htil$ is competitive not only against any unperturbed Lipschitz hypothesis,
but also against any $\eta^*$-perturbed hypothesis.
Moreover, the pointwise $\etstr$-perturbations might conspire 
to yield a lower empirical risk than unperturbed hypotheses.
Theorem~\ref{thm:risk-minimization} shows our approximate minimizer $\htil$ 
is competitive even against an ``optimally perturbed'' hypothesis $\hemp$.
\end{rem}

The rest of this section is devoted to proving 
Theorem~\ref{thm:risk-minimization} for $q=1$ 
(Sections \ref{sec:motivation} and \ref{sec:lp})
and for $q=2$ (Section \ref{sec:cp}).
We consider the $n$ observed samples as fixed values given as input to the algorithm
(as opposed to random samples),
so we will denote them $(x_i,y_i)$ instead of $(X_i,Y_i)$.
We will also restrict our attention to hypotheses for which 
$Q(\cdot) < 1$, since
otherwise our bounds are vacuous. 
Indeed, the minimizer $\hemp$ must satisfy this condition, 
which holds even for the flat hypothesis mapping
all points to $\frac{1}{2}$ 
(for sufficiently large $n$).

\subsection{Motivation and construction}
\label{sec:motivation}
We wish to find an optimal perturbed hypothesis 
$\hemp \in \H_{L^*\ge1}\oplus\turb{\etstr}$ minimizing $Q(\cdot)$.
Suppose that the Lipschitz and perturbation constants $L^*,\etstr$ 
of a minimizer $\hemp$ were known. Then the problem of computing 
both $\hemp$ and its empirical risk $\risk_n(\hemp,q)$ 
can be described as the following optimization program 
where variables $z_i$ representing the underlying smooth hypothesis of which
$\hemp$ is an $\etstr$-perturbation.
Note that the optimization program is a Linear Program (LP) when $q=1$ 
and a quadratic program when $q=2$.
\beqn \label{eq:program}
\framebox{ $
    \begin{array}{lll}
    \textrm{Minimize}   & \frac{1}{n} \sum_{i\in[n]} w_i^q   &    \\
    \textrm{subject to} & |z_i-z_j|    \le  L^* \cdot \rho(x_i,x_j) & \forall i,j \in [n] \\
            & w_i    \ge  |y_i - z_i| - \etstr  & \forall i \in [n] \\
                        & 0 \le z_i \le 1    & \forall i \in [n]    \\
                        & 0 \le w_i \le 1    & \forall i \in [n]
    \end{array}
$ }
\eeqn
After solving the program for variables $z_i$,
a minimizer $h^*$ can easily be derived: 
If solution $z_i$ is less than $y_i$ then 
$h^*(x_i) = \min \{z_i + \etstr , y_i \}$, 
and otherwise
$h^*(x_i) = \max \{z_i - \etstr , y_i \}$.
It follows that $\hemp$ could be computed by first obtaining $L^*$ and $\etstr$, 
and then solving the above program. 
However, both computing $L^*,\etstr$ and solving the program appear to be 
expensive computations, which motivates our approximate solution.
Note that supplying the LP with only a crude upper-bound on either $L^*$ or $\etstr$ could 
yield a hypothesis with large Lipschitz constant or perturbation, and potentially
poor generalization bounds.
We show below how to derive relatively tight estimates for $L^*,\etstr$, and in Section \ref{sec:lp}
we show how to solve the program quickly.

We first obtain a target perturbation constant $\etbar$ that ``approximates'' 
the unknown $\etstr$. In particular, we discretize candidate values of $\etbar$ to be
of the form $i \eta$ for integral $i \in [0,\lceil 1/\eta\rceil]$, and
search over all these values. (Recall that $\eta$ is the input to 
Theorem \ref{thm:risk-minimization}.) It follows
that there are only $O(1/\eta)$ candidates for $\etbar$, and that one
of these candidates satisfies $\etstr \le \etbar < \etstr + \eta$.

Next, we obtain a target Lipschitz constant $\Lbar$ that approximates $L^*$.
Recall that we have assumed that $L^* \ge 1$, and also have that
$L^* < n$, as otherwise the value of $Q(\hemp,L^*,\etstr)$ is necessarily greater than $1$.
We discretize the candidate values of $\Lbar$ to be of the form
$\left( 1+ \frac{\eta}{\ddim(\X)+1} \right)^i$ for integral $i \ge 0$, and 
search over all these values. 
It follows that there are only 
$O \left( \frac{\ddim{(\X)}}{\eta} \lognat n \right)$ 
discretized candidate values for $\Lbar$, and that one
of these candidates satisfies
$L^* \le \Lbar < \left( 1+ \frac{\eta}{\ddim(\X)+1} \right) L^*$.
We note that 
$$
\left( 1+ \frac{\eta}{\ddim(\X)+1} \right)^{\ddim(\X)+1}
\le e^\eta
\le 1+2\eta.
$$

Now replace $\etstr, L^*$ in program \eqref{eq:program} with approximations
$\etbar, \Lbar$, 
and let the hypothesis $\hbar$ be an optimal solution for the modified program;
this can only decrease the objective, 
i.e., $\risk_n(\hbar,q) \leq \risk_n(\hemp,q)$. 
Recall that
$Q(h^*) \le 1$, and so by the definition of $Q(\cdot)$ and the above bounds on 
$\etbar, \Lbar$ we have
\beq
Q(\hbar) < Q(\hemp) \cdot (1 + 2\eta) + 4\eta = Q(\hemp) + 6\eta.
\eeq
It remains to show that for each of the
$O \left( \frac{\ddim{(\X)}}{\eta^2} \lognat n \right)$
candidate pairs of $\Lbar$ and $\etbar$,
the modified linear program may be solved quickly (within
fixed
precision), 
which we do in Sections \ref{sec:lp} and \ref{sec:cp}.

\subsection{Solving the linear program}\label{sec:lp}

We show how to approximately solve the modified linear program, 
given target Lipschitz constant $\Lbar$ and perturbation parameter $\etbar$
(recall $\hbar$ is an optimal solution for this modified LP). 
Our solution will yield a hypothesis $\htil$ satisfying 
\beq
Q(\htil) \le Q(\hbar) + O(\eta).
\eeq

\paragraph*{Reduced constraints}
A central difficulty in obtaining a near-linear runtime for LP \eqref{eq:program} 
is that the number of constraints is $\Theta(n^2)$;
in particular, there are $\Theta(n^2)$ constraints of the form 
$|z_i-z_j| \le  \Lbar \cdot \rho(x_i,x_j)$.
We show how to reduce the number of these constraints (and only these constraints) 
to near-linear in $n$, namely, $\eta^{-O(\ddim (\X))}n$.
We will further guarantee that each of the $n$ variables $z_i$
appears in only $\eta^{-O(\ddim(\X))}$ constraints. Both these properties will prove useful 
for solving the program quickly. 

Recall that the purpose of the $\Theta(n^2)$ constraints is to ensure that 
the underlying hypothesis is smooth in the sense that the target Lipschitz 
constant is not violated between any pair of points. 
We show that 
this property can be approximately maintained with many fewer constraints.
To see this, consider a $1+\delta$ stretch spanner for the point set,
with spanner edge-set $E$. We claim that it suffices to enforce the
Lipschitz condition $\Lbar$ only on pairs that are endpoints in $E$:
Let $x_{k_1}$, $x_{k_j}$ be any pair that are not connecting by a single in $E$, 
and let
$x_{k_2},\ldots,x_{k_{j-1}}$ be the vertices encountered on the minimum
stretch path in $E$ connecting $x_{k_1}$ and $x_{k_j}$. Then by the stretch
guarantee of the spanner and the Lipschitz condition on its endpoints we have
\beq
\frac{|y_{k_1}-y_{k_j}|}{\rho(x_{k_1},x_{k_j})}
\le 
\frac{\sum_{i=1}^{j-1}|y_{k_i}-y_{k_{i+1}}|}{\rho(x_{k_1},x_{k_j})}
\le 
\frac{\sum_{i=1}^{j-1} \Lbar \rho(x_{k_i},x_{k_{i+1}})}{\rho(x_{k_1},x_{k_j})}
\le 
\frac{\Lbar (1+\delta)\rho(x_{k_1},x_{k_j})}{\rho(x_{k_1},x_{k_j})}
= 
(1+\delta)\Lbar.
\eeq

More formally, the constraints are reduced as follows:
The spanner described in Appendix \ref{sec:spanner} has stretch 
$1+\delta$, degree 
$\delta^{-O(\ddim(\X))}$
and hop-diameter
$c'\lognat n$ for some constant $c'>0$,
and can be computed quickly. 
Build this spanner for 
the observed sample points $\{x_i:\ i\in[n]\}$ with stretch $1+\frac{\eta^2}{2}$ 
(i.e., set $\delta = \frac{\eta^2}{2}$)
and retain a constraint in LP \eqref{eq:program} if and only if its two variables 
$z_i,z_j$ correspond to two vertices connected by a spanner edge 
(that is, edge $(x_i,x_j)$ is found in spanner's edge set $E$).
It follows from the bounded degree of the spanner that each variable appears in $\eta^{-O(\ddim(\X))}$ constraints, 
which implies that a total of $\eta^{-O(\ddim(\X))}n$ constraints are retained. 
Constructing the spanner (and thus the LP) takes time 
$\eta^{-O(\ddim(\X))}n \lognat n$.
The complete analysis of the Lipschitz guarantee appears below.

\paragraph*{Fast LP-solver framework} 
To solve the modified LP for fixed candidate values $\Lbar$ and $\etbar$, 
we utilize the framework presented by Young \cite{Y01} for LPs of the following form: Given 
non-negative matrices $P,C$, vectors $p,c$ and precision $\beta>0$, find a non-negative 
vector $x$ such that $Px \le p$ and $Cx \ge c$. Young shows that if there exists a 
feasible solution to the input instance, then a solution to a relaxation of the 
input program --- specifically, $Px \le (1+\beta)p$ and $Cx \ge c$ --- can be found in time $O(md (\lognat 
m)/\beta^2)$, where $m$ is the number of constraints in the program and $d$ is the 
maximum number of constraints in which a single variable may appear.
We may assume that constraints of the form $0 \le z_i \le 1$ and $0 \le w_i \le 1$ 
can be satisfied exactly:
Since $y_i \le 1$, we can always round down a solution variable to 1
without affecting the quality of the solution.

\paragraph*{Modifying the Lipschitz constraints}
In utilizing Young's framework for our problem, we encounter a difficulty that both the input matrices 
and output vector must be non-negative, while our LP \eqref{eq:program} has difference constraints. 
To bypass this limitation, we first consider the LP variables $z_i$, and for each one
introduce a new variable $0 \le \tilde{z}_i \le 1$ and two new constraints:
\begin{align*}
    z_i + \tilde{z}_i    & \le 1 ,   \\
    z_i + \tilde{z}_i    & \ge 1 . 
\end{align*}
These constrains require that $\tilde{z}_i = 1 - z_i$, but
by the relaxed guarantees of the LP solver, we have that in the returned solution
$1 - z_i \le \tilde{z}_i \le 1 - z_i + \beta$.
This technique allows us to introduce negated variables $-z_i$ into the linear 
program, at the loss of additive precision. 

Each retained spanner-edge constraint
$|z_i-z_j| \le \Lbar  \cdot \rho(x_i,x_j)$ 
is replaced by a pair of constraints
\begin{align*}
    z_i + \tilde{z}_j    & \le 1 + \Lbar  \cdot \rho(x_i,x_j) , \\
    z_j + \tilde{z}_i    & \le 1 + \Lbar  \cdot \rho(x_i,x_j)
  \end{align*}
Taken together, the above four constraints require that  
$1+|z_i-z_j| \le 1+ \Lbar \cdot \rho(x_i,x_j)$.
The modified program is found in \eqref{eq:program2}.

\beqn \label{eq:program2}
\framebox{ $
    \begin{array}{lll}
    \textrm{Minimize}   & \frac{1}{n} \sum_{i\in[n]} w_i   &    \\
    \textrm{subject to} & 1 \le z_i + \tilde{z}_i \le 1 & \forall i \in [n] \\
			& z_i + \tilde{z}_j    \le  \Lbar \cdot \rho(x_i,x_j) & \forall (x_i,x_j) \in E \\
			& w_i    \ge  |y_i - z_i| - \etbar  & \forall i \in [n] \\
                        & 0 \le z_i \le 1    & \forall i \in [n]    \\
			& 0 \le \tilde{z}_i \le 1   & \forall i \in [n]	\\
                        & 0 \le w_i \le 1    & \forall i \in [n]
    \end{array}
$ }
\eeqn

Below we will address the objective function and the related constraint
$w_i \ge  |y_i - z_i| - \etbar$, and show that they can be modified to
fit into Young's LP framework. But first, we will show that our modification
of the Lipshitz constraints, along with the approximate guarantees of the LP 
solver, still yield a hypothesis that is close to Lipschitz:

Recall that in the returned solution of the LP solver, 
$z_i \le 1$, and so necessarily $|z_i-z_j| \le 1$.
By the approximate guarantees of the LP solver, 
we have that in the returned solution to LP \eqref{eq:program2},
each spanner edge constraint will satisfy
\begin{align}
    |z_i-z_j| 
    & \le \min \{1, -1 + (1+\beta)[1 + \Lbar  \cdot \rho(z_i,z_j)] \} \notag \\
    & =   \min \{1, \beta + (1+\beta) \Lbar  \cdot \rho(z_i,z_j) \}    
    \label{eq:SpannerEdge}\\ 
    & \le 2\beta + \Lbar  \cdot \rho(z_i,z_j), \notag
\end{align}
where the last inequality follows by splitting into two cases,
depending on whether $\Lbar  \cdot \rho(z_i,z_j) \leq 1$.

To obtain a similar bound for point pairs not connected by a spanner edge:
Let $x_{1},\ldots,x_{k+1}$ be a $(1+\frac{\eta^2}{2})$-stretch $k$-hop 
spanner path connecting points $x_1$ and $x_{k+1}$, for $k \le c'\lognat n$;
then the stretch guarantee implies that  
$\sum_{i=1}^{k} \rho(x_i,x_{i+1}) \le (1+\frac{\eta^2}{2}) \rho(x_1,x_{k+1})$.
Using the triangle inequality and \eqref{eq:SpannerEdge},
and recalling the relaxed guarantees of the LP solver, we have that in the
returned solution to LP \eqref{eq:program2}
\begin{align*}
    |z_1-z_{k+1}| 
    & \le \min \{ 1, \textstyle \sum_{i=1}^{k} |z_i - z_{i+1}| \} \\
    & \le \min \{ 1, \textstyle \sum_{i=1}^{k} [\beta + (1+\beta) \Lbar  \cdot \rho(x_i,x_{i+1})]  \} \\
    & \le \min \{ 1, \beta k + (1+\beta)\Lbar  \cdot (1+{\eta^2}/{2}) \ \rho(x_1,x_{k+1}) \} \\
    & \le \min \{ 1, \beta (k + 1) + (1+{\eta^2}/{2}) \Lbar\cdot \rho(x_1,x_{k+1}) \} \\
    & \le \beta (c' \lognat n + 1) + \frac{\eta^2}{2} + \Lbar\cdot \rho(x_1,x_{k+1}),
  \end{align*}
where the fourth and fifth inequalities each follow by splitting into two cases.

Choosing $\beta = \frac{\eta^2}{2c' \lognat n +1}$,
we have that for all point pairs in the returned solution to LP \eqref{eq:program2}
\begin{align*}
    |z_i-z_j| 
    & \le \eta^2 + \Lbar  \cdot \rho(z_i,z_j)  .  
\end{align*}

Now let $h_z$ be the hypothesis mapping $x_i$ to the value of $z_i$ in the returned 
solution of the modified LP \eqref{eq:program2}.
In the original LP \eqref{eq:program}, the variables $z_i$ represented the 
$L^*$-Lipschitz underlying function. In the solver solution for 
LP \eqref{eq:program2}, the variables $z_i$ are a
$4\eta$-perturbation of an $\Lbar$-Lipschitz function:

\begin{lemma}\label{lem:hz}
  With $h_z$ defined as above,
$h_z \in \H_{\Lbar}\oplus\turb{4\eta}$.
\end{lemma}

\begin{proof}
Let us construct a function $\tilde{h}_z$ as follows:
Let $S$ be the sample points $\{x_i\}_{i\in[n]}$, and 
extract from $S$ an $\eta/\Lbar$-net $N$.%
\footnote{The notion of a net referred to here 
means that (i) the distance between every two points in $N$ is at least $\eta/\Lbar$;
and (ii) every point in $S$ is within distance $\eta/\Lbar$ from at least one point in $N$.
It can be easily constructed by a greedy process.}
For every net-point $v\in N$ set $\tilde{h}_z(v) = \frac{h_z(v)}{1+\eta}$. 
Then extend hypothesis $\tilde{h}_z$ from $N$ to all of the sample $S$ 
without increasing Lipschitz constant by using the McShane-Whitney extension 
theorem \cite{MR1562984,Whitney1934} for real-valued functions.\footnote{%
The McShane-Whitney extension theorem says that for every metric space $M$
and subset $N \subset M$, 
every $L$-Lipschitz $f:N\to\R$ 
can be extended to all of $M$ while preserving the $L$-Lipschitz condition.
}
This completes the description of $\tilde{h}_z$.

We first show that $\tsLip{\tilde{h}_z} \le \Lbar$. Indeed, for every two net-points
$v \neq v'\in N$ we have $\rho(v,v') \ge \eta/\Lbar$ and so

\[
    \begin{array}{lll}
    |\tilde{h}_z(v)-\tilde{h}_z(v')| 
    & = & \frac{|\tilde{h}_z(v)-\tilde{h}_z(v')|}{1+\eta}   \\
    & \le & \frac{\eta^2 + \Lbar \cdot \rho(v,v')}{1+\eta} \\
    & \le & \Lbar  \cdot \rho(v,v').
    \end{array}
\]
It follows that $\tilde{h}_z$ indeed satisfies the $\Lbar$-Lipschitz condition on the net-points.
By the extension theorem, 
$\tilde{h}_z$ achieves Lipschitz constant $\Lbar$ on all points of $S$. 

It remains to show that
$\tsnrm{h_z-\tilde{h}_z}_\infty \le 4 \eta$:
Consider any point $x \in S$ and its closest net-point $v \in N$; 
then $\rho(x,v) < \eta/\Lbar$ and we have 

\[
    \begin{array}{lll}
    |h_z(x) - \tilde{h}_z(x)| 
    & \le & |h_z(x) - h_z(v)| + |h_z(v) - \tilde{h}_z(v)| + |\tilde{h}_z(v) - \tilde{h}_z(x)|    \\
    & < & [\eta^2 + \Lbar \cdot \frac{\eta}{\Lbar}] 
        + [1-\frac{1}{1+\eta}] 
        + [\eta^2 + \Lbar \cdot \frac{\eta}{\Lbar}] \cdot \frac{1}{1+\eta}   \\
    & < & 4 \eta,
    \end{array}
\]
and we conclude that $h_z$ is a $4\eta$-perturbation of $\tilde{h}_z$.
\end{proof}

\paragraph*{Modifying the objective function}
We now turn to the constraints 
$w_i \ge |y_i - z_i| - \etbar$
and the objective function
$\frac{1}{n} \sum_{i \in [n]} w_i$.
Each LP constraint is replaced by a constraint pair 
\begin{align*}
  w_i + z_i           & \ge y_i - \etbar, \\
  w_i + \tilde{z}_i   & \ge 1 - y_i - \etbar,
\end{align*}
and together these require that
$w_i \ge |y_i - z_i| - \etbar$.
Note however that in the returned solution we are guaranteed only that 
$w_i \ge |y_i - z_i| - \etbar - \beta$.
Hence, the empirical error of the hypothesis is bounded by
$\beta + \frac{1}{n} \sum_{i\in[n]} w_i$
instead of 
$\frac{1}{n} \sum_{i\in[n]} w_i$.

The objective function is replaced by the constraint
\begin{align*}
    \frac{1}{n} \sum_{i \in [n]} w_i \le r  ,
\end{align*}
where $r$ itself it guessed by discretizing into multiples of $\eta$ ---
that is $\bar{r} = i\eta^2$ for integral $i \in [1,\lceil 1/\eta^2 \rceil]$ ---
which gives $O(1/\eta^2)$ candidate values for $r$.
By the discetization of $r$, 
the relaxed guarantees of the LP solver, 
and the above bound on the empirical error, 
the empirical error of the solution hypothesis $\htil$ 
is within an additive term
$\eta^2+\beta+\beta < 2\eta^2$ of optimal. 
The final program is found in \eqref{eq:program3}.

\beqn \label{eq:program3}
\framebox{ $
    \begin{array}{lll}
    \textrm{Find}   
                        & 0 \le z_i \le 1           & \forall i \in [n]        \\
                        & 0 \le \tilde{z}_i \le 1   & \forall i \in [n]        \\
                        & 0 \le w_i \le 1       & \forall i \in [n]        \\
    \textrm{subject to} & 1 \le z_i + \tilde{z}_i \le 1 & \forall i \in [n] \\
            & z_i + \tilde{z}_j    \le  \Lbar \cdot \rho(x_i,x_j) & \forall (x_i,x_j) \in E \\
                        & \sum_i w_i \le \bar{r}    & \forall i \in [n] \\
            & w_i + z_i \ge y_i - \etbar   & \forall i \in [n]     \\
            & w_i + \tilde{z}_i   \ge 1 - y_i - \etbar & \forall i \in [n]     \\
    \end{array}
$ }
\eeqn

\paragraph*{Correctness and runtime analysis}
Consider the choice of $\Lbar,\etbar$, closest to the values 
$L^*,\etstr$, and recall that for these values 
there exists a hypothesis $\hbar \in\H_{\Lbar\ge1}\oplus\turb{\etbar}$
satisfying 
\beq
Q(\hbar,\Lbar,\etbar) < Q(\hemp,L^*,\etstr) + 6\eta. 
\eeq
As shown above, running program \eqref{eq:program3} on this $\Lbar,\etbar$, we obtain a hypothesis
$\htil \in \H_{\Lbar}\oplus\turb{4\eta + \etbar}$
whose empirical risk is within an additive term $2\eta^2$ of the empirical
risk of the optimal $\hemp$. It follows that 
\beq
Q(\htil,\Lbar,4\eta + \etbar) 
\le Q(\hbar,\Lbar,\etbar) + 2\eta^2 + 4\eta
\le Q(\hemp,L^*,\etstr) + 11\eta .
\eeq
The result claimed in Theorem \ref{thm:risk-minimization} is achieved, 
up to scaling $\eta$, i.e., applying the above for $\eta=\eta_1/11$,
by exhaustively trying all pairs of candidates $\Lbar,\etbar$
and picking the pair that minimizes $Q(\cdot)$.

We turn to analyze the algorithmic runtime. Recall that 
the spanner can be constructed in time $O(\eta^{-O(\ddim(\X))}n \lognat n)$.
Young's LP solver \cite{Y01} is invoked on
$O \left( \eta^{-2}\ddim{(\X)} \lognat n \right)$
pairs of $\Lbar,\etbar$
and $O(1/\eta^2)$ candidate values of $\bar{r}$, for a total of 
$O \left( \eta^{-4} \ddim{(\X)} \lognat n \right)$
times. To determine the runtime 
per invocation, recall that each variable of the program appears in
$d = \eta^{-O(\ddim(\X))}$
constraints, implying that there are in total
$m = \eta^{-O(\ddim(\X))}n$ constraints. 
Since we set $\beta = O(\eta^2/\lognat n)$, we have that 
each call to the solver takes time
$O(md(\lognat m)/\beta^2) 
 \leq \eta^{-O(\ddim(\X))}n \lognat^2 n$, and the total runtime is
$\eta^{-O(\ddim(\X))}n \lognat^3 n$.
This completes the proof of Theorem \ref{thm:risk-minimization} for $q=1$.

\subsection{Solving the quadratic program}
\label{sec:cp}

We proceed to the case of a quadratic loss function,
i.e., $q=2$ in our original program \eqref{eq:program}.
A recent line of work on fast solvers for Laplacian systems
and for electrical flows, see e.g.\ \cite[Sections 3 and 11]{Vishnoi13},
provides powerful algorithms that can speed up
Laplacian-based machine-learning tasks \cite{ZGL03}.
However, these algorithms are not directly applicable here,
because our quadratic program \eqref{eq:program} contains
hard non-quadratic constraints to enforce a Lipschitz-constant bound $L^*$.
In fact, our program can be viewed as minimizing simultaneously
the $\ell_\infty$-Laplacian on the graph edges and 
some $\ell_2$-Laplacian related to the point values.
See also \cite{KRSS15} for a discussion of Lipschitz extension on graphs
and additional references.

Our approach is to modify the methodology we developed above for linear loss,
to cover the case of a quadratic loss function $\frac{1}{n} \sum_i w_i^2$.
Specifically, we introduce variables $v_i \ge w_i^2$, 
and replace the objective function with $\frac{1}{n} \sum_i v_i$. 
It remains to show how to model the constraints $v_i \ge w_i^2$.

First consider a parabola $y = x^2$, and note that a line 
$y = (2a)x - a^2$ 
is tangent to the parabola, intersecting it at $x=a$.
Hence, the constraint $v_i \ge w_i^2$ can be approximated by a constraint set
$v_i \ge (2a)w_i - a^2$ for $a = i\eta$ and integral 
$i \in [0,\lfloor 1/\eta \rfloor]$. 
These lines have slope in the range $[0,2]$, and so the approximation may
cause the value of $v_i$ to be underestimated by $2 \eta$. 
This is in addition to the previous underestimate of $w_i$, and by the above
scaling of $\eta$ this maintains the asymptotic error guarantee of the theorem.
Turning to the runtime analysis, the replacement of a single constraint by 
$O(1/\eta)$ new constraints does not change the asymptotic runtime.

\section{Approximate Lipschitz extension}
\label{sec:lipext}

In this section, we show how to evaluate our hypothesis on a new point.
We take the underlying smooth hypothesis on set $S$
implicit in Lemma~\ref{lem:hz} ---
call it $\tilde h_z(\cdot)$ ---
and we wish to evaluate a minimum Lipschitz extension of $\tilde h_z$ 
on a new point $x \notin S$. 
That is, denoting $S=\set{x_1,\ldots,x_n}$,
we wish to return a value $y = \tilde h_z(x)$
that minimizes
${\ds \max_{i\in[n]}  \frac{|y-\tilde h_z(x_i)|}{\rho(x,x_i)}  }$.
By the McShane-Whitney extension theorem, the extension of $\tilde h_z$ to the new 
point does not increase the Lipschitz
constant of $\tilde h_z$, and so the risk bound in Theorem~\ref{thm:main-risk} 
applies.\footnote{Theorems \ref{thm:LipExt} and
\ref{thm:main-risk} are ``local'' in the following sense.
At a test point $x$, Theorem~\ref{thm:LipExt} returns the value $h(x)$,
where $h:\X\to[0,1]$ is an $\eta$-perturbed $L$-Lipschitz function.
At a different test point $x'$, a different $h':\X\to[0,1]$ is evaluated.
There is no consistency requirement between $h$ and $h'$ ---
there need not exist {\em any} $\etbar$-perturbed $L$-Lipschitz function $h''$
such that $h''(x)=h(x)$ and $h''(x')=h'(x')$.
}

First note that the Lipschitz extension label $y$ of $x \notin S$ will be determined
by a pair of points of $S$: There exist 
points $x_i,x_j\in S$, one with label greater than $y$ and one with a label less than $y$, 
such that the Lipschitz constant of $x$ relative to each of these points 
(that is, 
$L
= \frac{\tilde h_z(x_i)-y}{\rho(x,x_i)}
= \frac{y-\tilde h_z(x_j)}{\rho(x,x_j)}
$)
is maximum over the Lipschitz constant of $x$ relative to any point in $S$.
Hence, $y$ cannot be increased or decreased without increasing the Lipschitz 
constant with respect to one of these points. 
Hence, an exact Lipschitz extension may be computed in $\Theta(n^2)$ time in brute-force fashion, 
by enumerating all point pairs in $S$, calculating the optimal 
Lipschitz extension for $x$ with respect to each pair alone, and then choosing the 
candidate value for $y$ with the highest Lipschitz constant. However, we demonstrate 
that an approximate solution to the Lipschitz extension problem can be obtained more 
efficiently.

\begin{theorem}\label{thm:LipExt}
An $\eta$-additive approximation to the Lipschitz extension problem on a 
function $f: S \rightarrow [0,1]$ can be 
computed in time $\paren{\oo\eta}^{-O(\ddim(\X))} \lognat n$.
\end{theorem}

\begin{proof}
The algorithm is as follows.
Round up all labels $f(x_i)$ to the nearest
multiple of
$j\eta/2$ (for any integer $0 \le j \le 2/\eta$), and call the new label function 
$\tilde{f}$. We seek the value of $\tilde{f}(x)$, the value at point $x$
of the optimal Lipschitz extension function $\tilde{f}$. 
Trivially, $f(x) \le \tilde f(x) \le f(x) + \eta/2$. 
Now, if we were given for each $j$ the point with label 
$j\eta/2$ that is the nearest neighbor of $x$ (among all points with this label), then 
we could run the brute-force algorithm described above on these $2/\eta$
points in time 
$O(1/\eta^{2})$ and compute $\tilde{f}(x)$. However, exact metric nearest neighbor search 
is potentially expensive, and so we cannot find these points efficiently. We  
instead find for each $j$ a point $x'\in S$ with label $\tilde{f}(x') = j \eta /2$ that is a 
$(1+\frac{\eta}{2})$-approximate nearest neighbor of $x$ among points with this label. 
(This can be done by presorting the points of $S$ into $2/\eta$ buckets based on their 
$\tilde{f}$ label, and once $x$ is received, 
running on each bucket a $(1+\frac{\eta}{2})$-approximate nearest neighbor 
search algorithm due to \cite{CG06} that takes $(1/\eta)^{O(\ddim(\X))} \lognat n$ time.)
We then run the brute force algorithm on these $2/\eta$ points in time
$O(1/\eta^{2})$.
The nearest neighbor search achieves approximation factor $1+\frac{\eta}{2}$,
implying a similar multiplicative approximation to $L$,
and thus also to $|y-f(x')| \leq 1$,
which means at most $\eta/2$ additive error in the value $y$.
We conclude that the algorithm's output
solves the Lipschitz extension problem within additive approximation $\eta$.
\end{proof}

\section{Risk bounds}
\label{sec:risk}

The algorithm in Section~\ref{sec:bv} produces 
a hypothesis $h:\X\to[0,1]$,\footnote{
Since $Y_i\in[0,1]$, there is no loss in assuming that the hypothesis also has this range;
this is trivially ensured by a truncation, which preserves the Lipschitz constant.
}
which is an $\etbar$-perturbation of some hypothesis in $H_L$
(the notation there was $\htil$ and $\Ltil$).
Recalling the definitions of empirical risk 
and expected risk 
in (\ref{eq:emprisk}) and (\ref{eq:exprisk}),
this section is devoted to proving that
with high probability,
$\risk(h,q)$ is not much greater than $\risk_n(h,q)$.
\begin{theorem}
\label{thm:main-risk}
Fix $q\in\set{1,2}$, $\eta\in(0,1]$,
and
$\etbar\in\set{\eta,2\eta,\ldots,\eta \lfloor 1/\eta \rfloor,1}$.
Then for all $\delta>0$, 
with probability at least $1-\delta$,
the following holds uniformly 
for all
$L\ge 1$
and all $h\in\H_{L}\oplus\turb{\etbar}$:
\beq
\risk(h,q) 
&\le&
\risk_n(h,q)
+
4(2q-1)\etbar
+
\sqrt{\frac{32
\lognat{\frac{8}{
(2q-1)\etbar
}}
}n}
\paren{\frac{16q^{3/2}
L
}{
(2q-1)\etbar
}}
^{1+\ddim(\X)}
\\
&+&\sqrt{\frac{\lognat\logtwo (2 
L^{1+\ddim(\X)}
)}n}
+3\sqrt{\frac{\lognat\frac4{\delta\eta}}{2n}}
.
\eeq
\end{theorem}

The proof of Theorem~\ref{thm:main-risk} proceeds in two conceptual steps.
We first bound the covering numbers for classes of Lipschitz functions 
(in Section~\ref{sec:cov-num})
and then use those to estimate Rademacher complexities
(in Section~\ref{sec:rad}).

\subsection{Covering numbers for Lipschitz function classes}
\label{sec:cov-num}
We begin by
obtaining complexity estimates
for Lipschitz functions in doubling spaces. 
In the conference version
\cite{DBLP:conf/simbad/GottliebKK13}
this was done in terms of the fat-shattering dimension, 
but here we obtain considerably simpler and tighter
bounds by direct control over the covering numbers.

The following variant of the classic 
``covering numbers by covering numbers''
estimate
\citet{MR0124720}
was proved together with Roi Weiss (cf. \cite[Lemma 2]{kon-weiss-2014}):
\begin{lemma}
\label{lem:cov-cov}
Let $\F_L$ be the collection of $L$-Lipschitz 
functions 
mapping
the metric space $(\X,\rho)$
to $[0,1]$.
Then the covering numbers of $\F_L$ may be estimated in terms of 
the covering numbers of $\X$:
\beq
\calN(\eps,\F_L,\nrm{\cdot}_\infty) \le \paren{\frac{8}{\eps}}^{ \calN(\eps/8L,\X,\rho)}.
\eeq
Hence, for doubling spaces with $\diam(\X)=1$,
\beq
\lognat \calN(\eps,\F_L,\nrm{\cdot}_\infty) 
\le 
\paren{\frac{16L}{\eps}}
^{\ddim(\X)}
\lognat
\paren{\frac{8}{\eps}}
.
\eeq
\end{lemma}
\bepf
Fix a covering of $\X$ consisting of 
$|N| = \calN(\eps/8L, \X , \rho)$ balls $\{U_1,\dots,U_{|N|}\}$ of radius 
$\eps' = \eps/8L$ 
and choose $|N|$ points $N = \{x_i \in U_i\}_{i=1}^{|N|}$.
We will construct an $\eps$-cover $\wh F=\set{\hat f_1,\ldots,\hat f_{|\hat F|}}$ as follows.
At every point $x_i\in N$, we choose $\hat f(x_i)$ to be some multiple of $2L\eps'=\eps/4$,
while maintaining
$\tsLip{\hat f}\le 2L$.
Construct a $2L$-Lipschitz  extension for $\hat f$ from $N$ to all over $\X$ 
(such an extension always exists, \citep{MR1562984,Whitney1934}).

We claim that every $f\in\F_L$ is close to some $\hat f\in \wh F$,
in the sense that $\tsnrm{f-\hat f}_\infty \leq \eps$.
Indeed, every point $x\in \X$ is $\eps'$-close to some point $x_N\in N$,
and since $f$ is $L$-Lipschitz and $\hat f$ is $2L$-Lipschitz,
\beq
  \tsabs{f(x)-\hat f(x)}
  &\leq& \tsabs{f(x)-f(x_N)} + \tsabs{f(x_N)-\hat f(x_N)} + \tsabs{\hat f(x_N)-\hat f(x)}\\
  &\leq& L\cdot \rho(x,x_N) + \eps/4 + 2L\cdot \rho(x,x_N)
  = \eps.
\eeq

It is easy to verify that $\tsabs{\hat F} \leq (8/\eps)^{\tsabs{N}}$,
since by construction, the functions $\hat f$ are determined by their values on $N$.
This provides a covering of $\F_L$ using $\tsabs{\hat F}$ balls of radius $\eps$.

The bound for doubling spaces follows immediately by applying 
the so-called doubling property (see for example \cite{KL04}) 
and the diameter bound, to obtain
\beq
\calN(\eps,\X,\rho)&\le& \paren{\frac2\eps}^{\ddim(\X)} .
\eeq
\enpf

Let us consider two additional properties that a metric space
$(\X,\rho)$ might possess:
\begin{enumerate}
\item
$(\X,\rho)$ is {\em connected} if for all $x,x'\in\X$ and all $\eps>0$,
there is a finite sequence of 
points $x=x_1,x_2,\ldots,x_m=x'$
such that 
$\rho(x_i,x_{i+1})<\eps$
for all $1\le i<m$.
\item
$(\X,\rho)$ is {\em centered} if 
for all $r>0$ and
all $A \subset \X$ with
$\diam(A) \leq 2r$, there exists a point $x \in \X$ such that $\rho(x, a) \leq r$ for all $a \in A$.
\end{enumerate}
The estimate in Lemma~\ref{lem:cov-cov} may be improved for doubling spaces that are additionally
connected and centered, as follows.

\begin{lemma}[\cite{MR0124720}]
If $(\X,\rho)$ is connected and centered, then,
for constant $\ddim(\X)$,
\label{lem:tight_CovNum}
\beq
\lognat\calN(\eps,\F_L, \nrm{\cdot}_\infty)
= O\paren{
\paren{\frac{L}{\eps}}^{\ddim(\X)}
+
\lognat\paren{\oo\eps}
}
.
\eeq
\end{lemma}

\subsection{Rademacher complexities}
\label{sec:rad}
The (empirical)
{\em Rademacher complexity} 
\cite{DBLP:journals/jmlr/BartlettM02,MR1892654}
of 
a collection of functions $\F$
mapping some set $\calZ$ to $\R$
is 
defined,
with respect to a sequence
$Z=(Z_i)_{i\in[n]}\in \calZ^n$,
 by
\beqn
\label{eq:rad}
{\rad}_n(\F;Z) = 
\E
\sqprn{
\sup_{f\in\F} 
\oo{n}\sum_{i=1}^n \sigma_i f(Z_i)
},
\eeqn
where the expectation is over
the $\sigma_i$, which are 
iid
with $\P(\sigma_i=+1)=\P(\sigma_i=-1)=1/2$.

To any collection $\G$ of hypotheses 
mapping $\X$ to $\R$,
we associate
the {\em $q$-loss class},
whose members
map $\X\times\R$ to $\R$.
The latter is denoted by
$\qloss{\G}$
and defined to be
\beqn
\label{eq:loss-class}
\qloss{\G}=
\set{ f:(x,y)\mapsto\abs{g(x)-y}^q ; g\in\G}.
\eeqn

It will also be convenient to define the auxiliary metric space
$(\calZ,d_q)$, where $\calZ=\X\times[0,1]$ and
\beqn
\label{eq:dq}
d_q((x,y),(x',y')) = \paren{ \rho(x,x')^q + \abs{y-y'}^q}^{1/q}.
\eeqn

Let us recall the relevance of Rademacher complexities to risk estimates 
\cite[Theorem 3.1]{mohri-book2012}:
for every $\delta>0$, we have that,
with probability at least $1-\delta$,
\beqn
\label{eq:rade-risk}
\risk(g,q) \le \risk_n(g,q) + 2\rad_n(\qloss{\G};Z)
+3\sqrt{\frac{\lognat(2/\delta)}{2n}},
\eeqn
holds uniformly over all $g\in\G$,
where $Z=(X_i,Y_i)_{i\in[n]}$ is the training sample.

The following simple and well-known
estimate of Rademacher complexity is obtained via covering numbers;
see, e.g., \cite[Theorem 1.1]{bartlett-notes} 
for the proof of a closely related fact.
\begin{lemma}
\label{lem:massart}
For all function classes $\F\subset[0,1]^\calZ$,
all $Z\in\calZ^n$,
 and all $\eps>0$,
\beq
\rad_n(\F;Z) \le \eps+\sqrt{\frac{2\lognat\calN(\eps,\F,\nrm{\cdot}_\infty)}n}.
\eeq
\end{lemma}

Having reduced the problem to one of estimating covering numbers, we would like to invoke 
results from Section~\ref{sec:cov-num}, such as Lemma~\ref{lem:cov-cov}. 
The following result sheds 
light on the relation
between $\H_L$
and its loss class.
Its proof appears in Appendix~\ref{sec:rade-proofs}.
\begin{lemma}
\label{lem:f2h}
Let $(\calZ,d_q)$ be as defined in (\ref{eq:dq}) and
$q\in\set{1,2}$.
The following relations hold:
\bit
\item[(i)]
  if
  $f\in\qloss{\H_{L\ge1}}$
  with witness $h\in\H_{L\ge1}$,
  then
  $\Liprho{f}{d_q}\le q^{3/2}\Liprho{h}{\rho}$,
\item[(ii)] $\ddim(\calZ,d_q) \le 2+2\ddim(\X,\rho)$.
\eit
\end{lemma}

We are ready to prove an ``unperturbed'' version of Theorem~\ref{thm:main-risk}, as follows.
\begin{theorem}
\label{thm:risk-unperturbed}
For $q\in\set{1,2}$,
$L\ge1$,
and 
$0<\delta<1$, 
with probability at least $1-\delta$,
the following holds uniformly over all $h\in\H_{L}$:
\beqn
\label{eq:risk-unperturbed}
\risk(h,q)
-
\risk_n(h,q)
&\le& 
3\sqrt{\frac{\lognat \frac{2}{\delta}}{2n}}
+2\inf_{\eps>0}\sqprn{
\eps
+\sqrt{\frac{2
{\lognat{\frac{8}{\eps}}}
}n}
\paren{\frac{16q^{3/2}{L}
}{
\eps
}}
^{1+\ddim(\X)}
}
.
\eeqn
\end{theorem}
\bepf
Let $Z=(X_i,Y_i)_{i\in[n]}$ be the training sample and fix some $L\ge1$ and $\eps>0$.
We begin by applying (\ref{eq:rade-risk}) to $\G=\H_L$,
and get that with probability at least $1-\delta$,
uniformly for all hypotheses $h\in\H_L$,
\beq
\risk(h,q) &\le& \risk_n(h,q) + 2
\rad_n(\qloss{\paren{
\H_L
}};Z)
+3\sqrt{\frac{\lognat(2/\delta)}{2n}}
.
\eeq
Further,
\beqn
\rad_n(\qloss{\paren{
\H_L
}};Z) &\le& 
\eps
+\sqrt{\frac{2\lognat\calN(
\eps
,
\qloss{\H_L}
,\nrm{\cdot}_\infty)}n} \nonumber\\
&\le& 
\eps
+\sqrt{\frac{2
\paren{\frac{16q^{3/2}{L}}{
\eps
}}
^{2+2\ddim(\X)}
\lognat{\frac{8}{
\eps
}}
}n} \nonumber\\
&=&
\label{eq:eps+sqrt}
\eps
+\sqrt{\frac{2
{\lognat{\frac{8}{\eps}}}
}n}
\paren{\frac{16q^{3/2}{L}}{
\eps
}}
^{1+\ddim(\X)}
,
\eeqn
where the first inequality follows from Lemma~\ref{lem:massart}
and the second one by applying the covering number estimate in Lemma~\ref{lem:cov-cov}
to $\qloss{\H_L}$, after the appropriate ``conversion'' of Lipschitz constants and doubling dimensions
furnished by Lemma~\ref{lem:f2h}.
\enpf

For completeness,
we relate the empirical risk to the optimal risk.
\begin{corollary}
\label{cor:excess-risk}
Fix $q\in\set{1,2}$,
$L\ge1$,
and 
$0<\delta<1$,
and define
\beq
\hat \risk_n(q) &:=& \inf_{h\in\H_L} R_n(h,q), \\
 \risk^*(q) &:=& \inf_{h\in\H_L} R(h,q).
\eeq
Then
\beq
\hat \risk_n(q)
-
\risk^*(q)
&\le&
3\sqrt{\frac{\lognat \frac{2}{\delta} }{2n}}
+
2\inf_{\eps>0}\sqprn{
\eps
+\sqrt{\frac{2
{\lognat{\frac{8}{\eps}}}
}n}
\paren{\frac{16q^{3/2}{L}
}{
\eps
}}
^{1+\ddim(\X)}
}
\eeq
holds
with probability at least $1-\delta$.
\end{corollary}
\begin{proof}
  It will be convenient to denote the right-hand side
  of (\ref{eq:risk-unperturbed}) by $\Delta(\delta)$
  and to assume, without loss of generality,
  the existence of minimizers $\hat h_n$ and $h^*$
  of $\risk_n(\cdot,q)$ and $\risk(\cdot,q)$, respectively, over $\H_L$;
  this is justified via a standard approximation argument.
  A standard symmetrization argument
  (e.g., swapping $\Phi(S)$ and $\Phi(S')$ in
  \cite[Eq. (3.6)]{mohri-book2012})
  shows that
  the estimate of Theorem~\ref{thm:risk-unperturbed}
  holds in the other direction as well:
  \beq
  \P\paren{
    \sup_{h\in\H_L}{
      \risk_n(h,q)
      -
      \risk(h,q)
    }>\Delta(\delta)
  }\le\delta.
  \eeq
  Now using the fact that $\hat h_n$ is a minimizer, 
  \beq
  \hat \risk_n(q)-\risk^*(q)
  &=&
  \risk_n(\hat h_n,q)-\risk(h^*,q) \\
  &\le&
  \risk_n(h^*,q)-\risk(h^*,q) , 
  \eeq
  whence
  \beq
  \P\paren{
    \hat \risk_n(q)-\risk^*(q)
    \le \Delta(\delta)
    }\ge1-\delta.
  \eeq
\end{proof}
  
To extend
Theorem~\ref{thm:risk-unperturbed} to perturbed hypotheses, 
we will need the following decomposition,
whose proof appears in Appendix~\ref{sec:rade-proofs}.
\begin{lemma}
\label{lem:f2h-turb}
If 
$\eta>0$
and
$\H$ is any collection of functions mapping $\X$ to $[0,1]$,
then
\beq
\qloss{\paren{\H\oplus\turb{\eta}}}
\subseteq 
(\qloss{\H})\oplus\turb{(2q-1)\eta}
.
\eeq
\end{lemma}

\begin{corollary}
\label{cor:riskL-turb}
For $q\in\set{1,2}$,
$L\ge1$,
$\eta>0$, 
and
$0<\delta<1$, 
with probability at least $1-\delta$,
the following holds uniformly over all $h\in\H_{L}\oplus\turb{\eta}$:
\beq
\risk(h,q) &\le& 
4(2q-1)\eta
+
\risk_n(h,q)
+\sqrt{\frac{8
{\lognat{\frac{8}{(2q-1)\eta}}}
}n}
\paren{\frac{16q^{3/2}{L}}{
(2q-1)\eta
}}
^{1+\ddim(\X)}
+3\sqrt{\frac{\lognat(2/\delta)}{2n}}
.
\eeq
\end{corollary}
\bepf
For any sequence $Z=(X_i,Y_i)_{i\in[n]}$,
we have
\beq
\rad_n(\qloss{\paren{\H_L\oplus\turb{\eta}}};Z)
&\le&
\rad_n(
(\qloss{\H_L})\oplus\turb{(2q-1)\eta}
;Z)\\
&\le&
\rad_n(\qloss{\H_L};Z) + (2q-1)\eta \\
&\le&
2(2q-1)\eta
+\sqrt{\frac{2
{\lognat{\frac{8}{(2q-1)\eta}}}
}n}
\paren{\frac{16q^{3/2}{L}}{
(2q-1)\eta
}}
^{1+\ddim(\X)},
\eeq
where the first inequality follows from Lemma~\ref{lem:f2h-turb},
the second from the sub-additivity of Rademacher complexities
(\cite[Theorem 3.3]{CambridgeJournals:8212764}),
and the third from (\ref{eq:eps+sqrt})
(with $\eps=(2q-1)\eta$).
Invoking
(\ref{eq:rade-risk})
to bound the risk in terms of $\rad_n$
completes the proof.
\enpf

\hide{
The final ingredient an estimate on 
$\rad_n(\qloss{\H_L})$:
\begin{lemma}
\label{lem:radHL}
For $n\ge1$,
$q\in\set{1,2}$, $L\ge0$, and any sequence $Z=(X_i,Y_i)_{i\in[n]}$,
\beq
\rad_n(\qloss{\H_L};Z) &\le& 
2q^{3/2}\max\set{1,L}\paren{\frac{\lognat8}{n}}^{1/(3+2\ddim(\X,\rho))}.
\eeq
\end{lemma}
}

\bepf[Proof of Theorem~\ref{thm:main-risk}]
In light of 
Corollary~\ref{cor:riskL-turb}, 
it only remains to extend the risk bound from a fixed $(L,\etbar)$
to hold uniformly over all $L\ge1$
and $\etbar\in\set{\eta,2\eta,\ldots,\eta \lfloor 1/\eta \rfloor,1}$.
This is carried out via a standard 
stratification argument,
such as the one given in \cite[Theorem 4.5]{mohri-book2012}. 
To stratify over $L$, take
$\rho\inv=
L
^{1+\ddim(\X)}
$ 
in (4.42) ibid., we have that
with probability at least $1-\delta$,
\beq
\risk(h,q) \le \risk_n(h,q) 
+ \frac4\rho\rad_n( \qloss{\H_1};Z)
+\sqrt{\frac{\lognat\logtwo\frac2\rho}n}
+3\sqrt{\frac{\lognat\frac2\delta}{2n}}
\eeq
holds uniformly over all $h\in\bigcup_{L\ge1}\H_{L}$.
As in the proof of Corollary~\ref{cor:riskL-turb}, 
the cumulative effect of $\etbar$-perturbation is 
an additive error term of $4(2q-1)\etbar$.
To stratify over $\etbar$, notice that
$\etbar$ is chosen from an a-priori fixed
set of size $\ceil{1/\eta}\le2/\eta$ ---
and so taking a union bound amounts to replacing $\delta$
by $\delta\eta/2$.
\enpf

\begin{rem}
\label{rem:adapt}
The runtime guarantees of Theorems \ref{thm:risk-minimization}
and \ref{thm:LipExt},
as well as the risk bound of Theorem~\ref{thm:main-risk},
all depend exponentially on the doubling dimension of the metric space $\X$,
hence even a modest dimensionality reduction yields dramatic savings 
in algorithmic and sample complexities. 
This was exploited in \cite{GottliebKK13-tcs},
which develops a technique 
that may roughly be described as a metric analogue of PCA.
A set $X=\set{x_1,\ldots,x_n}\subset\X$ inherits the metric $\rho$ of $\X$ and hence 
$\ddim(X)\le2\ddim(\X)$
is well-defined
\cite[Lemma 6.6]{DBLP:journals/corr/FeldmannFKP15}.
Let us say that $\tilde X=\set{\tilde x_1,\ldots, \tilde x_n}\subset\X$ is an
$(\alpha,\beta)$-{\em perturbation} of $X$ if
$\tfrac1n \sum_{i=1}^n\rho(x_i,\tilde x_i)\le\alpha$ and $\ddim(\tilde X)\le\beta$.
Intuitively, the data is ``essentially'' low-dimensional if it admits
an $(\alpha,\beta)$-{perturbation} with small $\alpha,\beta$,
which leads to improved Rademacher estimates.

The data-dependent nature of ${\rad}_n$ was used in \cite{GottliebKK13-tcs} 
to develop generalization bounds that can exploit data that is 
essentially low-dimensional in the above sense.
That paper dealt with the binary classification setting,
and the technique was applied to the multiclass case by
\cite{kon-weiss-2014}.
The same dimensionality reduction technique applies just as directly 
in our context of regression 
(the proof is deferred to Appendix~\ref{sec:rade-proofs}).
\begin{theorem}
\label{thm:rad-dim-red}
Let $Z=(X,Y)\in\X^n\times[0,1]^n$ be the training sample
and suppose that $X$ admits
an
$(\alpha,\beta)$-perturbation $\tilde X$.
Then, for $L\ge1$,
\beq
{\rad}_n(
\qloss{\paren{\H_L\oplus\turb{\eta}}}
;Z)
&\le &
{2(2q-1)\eta}
+
{q^{3/2}L\alpha}
+
\sqrt{\frac{2
{\lognat{\frac{8}{(2q-1)\eta}}}
}n}
\paren{\frac{16q^{3/2}{L}}{
(2q-1)\eta
}}
^{1+\beta}
.
\eeq
\end{theorem}

A key feature of the bound above is that it does not explicitly
depend on $\ddim(\X)$ (the dimension of the ambient space) or
even on $\ddim(X)$ (the dimension of the data).

\hide{
A central feature of the bound above is that it does not depend on $\ddim(\X)$
(the dimension of the ambient space)
or even on $\ddim(X)$
(the dimension of the data).
Note the inherent tradeoff between the distortion $\alpha$
and dimension $\beta$, with some non-trivial $(\alpha^*,\beta^*)$ minimizing the right-hand side
of the bound. Although computing the optimal 
$(\alpha^*,\beta^*)$ seems computationally difficult, 
\citet{GottliebKK13-tcs} 
were able to obtain an 
efficient
$(O(1),O(1))$-{\em bicriteria approximation}.
Namely, their algorithm computes an 
$\tilde\alpha\le c_0\alpha^*$
and
$\tilde\beta\le c_1\beta^*$,
with the corresponding perturbed set $\tilde X$,
for universal constants $c_0,c_1$,
with a runtime of
$2^{O(\ddim(X))}n\lognat n+O(n\lognat^5n)$.

The optimization routine over $(\alpha,\beta)$ may then be embedded inside
our SRM procedure described in Theorem~\ref{thm:risk-minimization}.
The end result will be a nearly optimal (in the sense of (\ref{eq:risk-minimization}))
Lipschitz constant ${L}$
as well as $(\tilde\alpha,\tilde\beta)$, which induce the perturbed set $\tilde Z=(\tilde X,Y)$.
To evaluate our hypothesis on a test point, we may invoke the $\eta$-approximate Lipschitz extension
routine from Theorem~\ref{thm:LipExt}.
This involves a precomputation of time complexity
$\eta^{-O(\tilde\beta)}n \lognat^3 n$
after which values on new points are computed in
$\eta^{-O(\tilde\beta)} \lognat n$.
Note that the evaluation time complexity depends only on the ``intrinsic dimension''
$\tilde\beta$ of the data, rather than the ambient metric space dimension.

\knote{please check that I didn't cheat here with the runtimes}
}

\end{rem}

\subsection*{Acknowledgements}
We thank Larry Wasserman for helpful comments on the manuscript, and
Robert Schapire for useful discussion and feedback.

\bibliographystyle{plain}

\bibliography{mybib}

\appendix

\subsection{A small-hop spanner}\label{sec:spanner}

In this section, we prove the following theorem.
See Section \ref{sec:tech} for the definition of a spanner.

\begin{theorem}\label{thm:spanner}
Every finite metric space $\X$ on $n$ points
admits a $(1+\delta)$-stretch spanner with 
degree $\delta^{-O(\ddim(\X))}$ (for $0 < \delta \le \frac{1}{2}$) 
and hop-diameter $O(\lognat n)$,
that can be constructed in time $\delta^{-O(\ddim(\X))}n\lognat n$.
\end{theorem}

Gottlieb and Roditty \cite{GR08b} presented for general metrics a $(1+\delta)$-stretch spanner 
with degree $\delta^{-O(\ddim (\X))}$
and construction time $\delta^{-O(\ddim(\X))}n\lognat n$, but this spanner has potentially large hop-diameter. Our 
goal is to modify this spanner to have low hop-diameter, without significantly increasing the spanner 
degree. Now, as described in \cite{GR08b}, the points of $\X$ are arranged in a tree of degree 
$\delta^{-O(\ddim (\X))}$, and a spanner path is composed of three consecutive parts: (a) a path 
ascending the edges of the tree; (b) a single edge; and (c) a path descending the edges of the tree. We 
will show how to decrease the number of hops in parts (a) and (c). Below we will prove the following 
lemma.

\begin{lemma}\label{lem:tree}
Let $T$ be a tree containing directed child-parent edges ($n=|T|$), and let $p$ be the degree of $T$. 
Then $T$ may be augmented with directed descendant-ancestor edges to create a DAG $G$ with the following 
properties:
(i) $G$ has degree $p+3$; and 
(ii) The hop-distance in $G$ from any node to each of its ancestors is $O(\lognat n)$.
\end{lemma}

Note that Theorem \ref{thm:spanner} is an immediate consequence of Lemma \ref{lem:tree} applied to the 
spanner of \cite{GR08b}. It remains only to prove Lemma \ref{lem:tree},
for which we will need the following preliminary lemma.

\begin{lemma}\label{lem:path}
Consider an ordered path on nodes $x_1,\ldots,x_n$. Let these nodes be assigned positive weights $w_i 
= w(x_i)$, and let the weight of the path be $W = \sum_{i=1}^n w(x_i)$. there exists a DAG $G$ on 
these nodes with the following properties:
\begin{enumerate}
\item
Edges in $G$ always point to the antecedent node in the ordering.
\item
The hop-distance from any node $x_i$ to the root node $x_1$ is not more than 
$O(\lognat \frac{W}{w_i})$.
\item
The hop-distance from any node $x_i$ to an antecedent $x_j$ is not more than 
$O(\lognat \frac{W}{w_i} +\lognat \frac{W}{w_j})$.
\item
$G$ has degree 3.
\end{enumerate}
\end{lemma}

\begin{proof}[Proof of Lemma~\ref{lem:path}]
The construction is essentially the same as in the biased skip-lists of Bagchi et al.\ \cite{Bagchi2005}.
Let $x_1$ and $x_n$ be the left and right {\em end nodes} of the path, and let the other nodes be the 
{\em middle nodes}. Partition the middle nodes into two child subpaths $\{x_2,\ldots,x_i\}$ (the left 
child path) and $\{x_{i+1},\ldots,x_{n-1}\}$ (the right child path), where $x_i$ is chosen so that the 
weight of the middle nodes of each child path is not more than half the weight of the middle nodes of 
the parent path. (If the parent path has three middle nodes or fewer, then there will be a single child 
path.) The child paths are then recursively partitioned, until the recursion reaches paths with no 
middle nodes.

The edges are assigned as follows. A right end node of a path has two edges leaving it. One points to 
the left end node of the path (unless the path has only one node). The other edge points to the right 
end node of the right (or single) child path. A left end node of a path has one edge leaving it: 
If this 
path is a right child path, 
the edge points to the left sibling 
{path}'s right end node. If this path is 
a left or single child path, then the edge points to the parent's left end node. The lemma follows via 
standard analysis. 
\end{proof}

We are now ready to prove Lemma \ref{lem:tree}, 
which would conclude the proof of Theorem \ref{thm:spanner}.

\begin{proof}[Proof of Lemma~\ref{lem:tree}]
Given tree $T$, decompose $T$ into {\em heavy paths}: A heavy path is one that begins at the root and 
continues with the heaviest child, the child with the most descendants. In a heavy path decomposition, 
all off-path subtrees are recursively decomposed. For each heavy path, let the weight of each node in 
the path be the number of descendant nodes in its off-path subtrees. For each heavy path, we build the 
weighted construction of Lemma \ref{lem:path}. 

Now, a path from node $u \in T$ to $v \in T$ traverses a set of at most $\lceil \lognat n \rceil$ 
heavy paths, say paths $P_1, \ldots, P_j$. The number of hops from $u$ to $v$ is bounded by 
$O(\lognat \frac{w(P_1)}{w(u)} 
+ \left( \sum_{i=1}^{j} \lognat \frac{w(P_{i-1})}{w(P_{i})} \right) 
+ \lognat \frac{n}{w(v)})
= O(\lognat n$), 
and the degree of $G$ is at most $p+3$. 
\end{proof}

\subsection{Rademacher-complexity proofs}
\label{sec:rade-proofs}
\hide{
Fix an $\eps>0$ and define the {\em projection} of $\F$ onto $Z=(Z_1,\ldots,Z_n)$:
\beq
\F(Z) = \set{ (f(Z_1),\ldots,f(Z_n)) : f\in\F}.
\eeq
Take
$\hat F=\set{\hat f_1,\ldots,\hat f_N}$ to be any and $\eps$-cover of $\F$ and notice
that $\hat F$ is in particular also an $\eps$-cover of $\F(Z)$.
For each $f\in\F$, define $\hat\delta_f=f-f^*$, where $f^*$ a minimizer of $\tsnrm{f-f'}_\infty$
over $f'\in\hat F$.
Then
\beq
\rad_n(\F;Z)
& = &
\E\sqprn{\sup_{f\in\F} \oo{n}\sum_{i=1}^n \sigma_i f(Z_i)
} \\
&=&
\E\sqprn{\sup_{f\in\F} \oo{n}\sum_{i=1}^n \sigma_i 
(f^*(Z_i)-\hat\delta_f(Z_i))
} \\
&\le&
\E\sqprn{\max_{\hat f\in\hat F} \oo{n}\sum_{i=1}^n \sigma_i 
f^*(Z_i)
} 
+
\E\sqprn{\sup_{f\in\F} \oo{n}\sum_{i=1}^n \sigma_i 
\abs{\hat\delta_f(Z_i)}
}.
\eeq
Since $|\hat F|=N<\infty$ and $\max_{\hat f\in F}\tsnrm{\hat f}_2\le\sqrt n$,
Massart's finite class lemma \cite[Theorem 3.3]{mohri-book2012} applies:
\beq
\E\sqprn{\max_{\hat f\in\hat F} \oo{n}\sum_{i=1}^n \sigma_i 
f^*(Z_i)
\,} 
&\le& \sqrt{\frac{2\lognat N}{n}}.
\eeq
Now by construction
$\tsnrm{\hat\delta_f}_\infty\le\eps$,
and hence
\beq
\E\sqprn{\sup_{f\in\F} \oo{n}\sum_{i=1}^n \sigma_i 
\abs{\hat\delta_f(Z_i)}
}
\le\eps,
\eeq
whence the claim follows.
}

\bepf[Proof of Lemma~\ref{lem:f2h}]
Suppose that $h:\X\to[0,1]$
with $\Lip{h}=L$,
$f(x,y)=\abs{h(x)-y}^q$, and $(\calZ,d_q)$ is the metric space defined in
(\ref{eq:dq}).
To prove (i), we consider the cases $q=1,2$ separately. For $q=1$,
\beqn
\abs{f(x,y)-f(x',y')} &=&
\abs{
\abs{h(x)-y}
-
\abs{h(x')-y'}
}
\nonumber\\
&\le&
\abs{
\paren{h(x)-y}
-
\paren{h(x')-y'}
}\nonumber\\
&\le&
\abs{h(x)-h(x')}
+\abs{y-y'}
\nonumber\\
&\le&
L\rho(x,x')+\abs{y-y'}\nonumber\\
&\le&
\max\set{1,L}\paren{\rho(x,x')+\abs{y-y'}}\label{eq:max1L}\\
&=&
\max\set{1,L}d_1((x,y),(x',y')),\nonumber
\eeqn
which proves the claim for this case.
Now consider the case $q=2$ and recall the following basic fact:
if $\varphi$ maps $E\subset\R^k$ to $\R$, then
\beq
\sup_{x\neq x'\in E}\frac{ \abs{\varphi(x)-\varphi(x')} }{ \nrm{x-x'}_2}
\le \sup_{z\in E}\nrm{\grad\varphi(z)}_2.
\eeq
Let us take $\varphi[0,1]^2\to\R$ to be $\varphi(h,y)=(h-y)^2$, which satisfies
\beq
\max_{(h,y)\in[0,1]^2} \nrm{\grad\varphi(h,y)}_2=2^{3/2}.
\eeq
It follows that
\beq
\abs{f(x,y)-f(x',y')} &=&
\abs{
(h(x)-y)^2
-
(h(x')-y')^2
}
\\
&\le&
2^{3/2}\paren{(h(x)-h(x'))^2+(y-y')^2}^{1/2}
\\&\le&
2^{3/2}\paren{
(L\rho(x,x'))^2
+(y-y')^2}^{1/2}
\\&\le&
2^{3/2}\max\set{1,L}d_2((x,y),(x',y')),
\eeq
which completes the proof of (i).

To prove (ii), we will show that
\beqn
\lambda(\calZ,d_q)\le 4\lambda(\X,\rho)^2,
\eeqn
where $\lambda(\cdot)$ is the doubling constant of a given metric space.
Consider the case $q=1$, put $a=\lambda(\X,\rho)$,
and fix any $d_1$-ball $B\subset\calZ$ with diameter $r$.
Define the coordinate projections $\pi_1:\calZ\to\X$ and $\pi_2:\calZ\to[0,1]$
in the obvious way and assume without loss of generality that $\pi_2(B)\subset[b,b+r)$.
Now partition $B$ into $4$ subsets based on the second coordinate:
\beq
B_i=\set{z\in B: \pi_2(z)\in\left[b+\frac{i}4,b+\frac{i+1}4\right)}
\eeq
for $i=0,1,2,3$.

By definition of the doubling constant, each $\pi_1(B_i)\subset\X$ can be covered by $a^2$ balls $V\subset\X$
of diameter at most $r/4$ under the metric $\rho$. It follows by construction that each $B_i$
can be covered by $a^2$ sets of the form
\beq
V\times[b+i/4,b+(i+1)/4),
\eeq
each of $d_1$-diameter at most $r/2$. Hence, any ball in $\calZ$ can be covered by $4a^2$ balls of half the diameter,
and so the claim is proved for $q=1$. To handle the case $q=2$, observe that
\beq
d_2((x,y),(x',y'))\le d_1((x,y),(x',y'))
\eeq
for all $(x,y),(x',y')\in\calZ$.
This proves (ii).
\enpf

\bepf[Proof of Lemma~\ref{lem:f2h-turb}]
Let $\tilde h(x)=h(x)+\delta(x)$, with $\nrm{\delta}_\infty\le\eta$
be an $\eta$-perturbed version of $h$,
with the corresponding $\tilde f(x,y)=\tsabs{h(x)-y}^q$.
Consider the case $q=1$. Then
\beq
\abs{f(x,y)-\tilde f(x,y)}
&=&
\abs{
\tsabs{h(x)-y}
-
\tsabs{\tilde h(x)-y}
}\\
&\le&
\abs{
(h(x)-y)
-
(\tilde h(x)-y)
}\\
&=&
\tsabs{h(x)-\tilde h(x)} = \abs{\delta(x)} \le\eta,
\eeq
which proves this case.
For $q=2$, we have
\beq
\abs{f(x,y)-\tilde f(x,y)}
&=&
\abs{
(h(x)-y)^2
-
(\tilde h(x)-y)^2
}\\
&=&
\abs{
[h(x)+\delta(x)-y]^2
-
[h(x)-y]^2
}\\
&=&
\delta(x)\abs{
2h(x)+\delta(x)-2y
}\le3\eta,
\eeq
since $0\le h,y,\delta\le1$.
\enpf

\bepf[Proof of Theorem~\ref{thm:rad-dim-red}]
Put $\tilde Z=(\tilde X,Y)$.
For $X_i\in X$, $\tilde X_i\in\tilde X$, 
and $f\in \qloss{\H_L}$,
define 
$\delta_i(f)=f(X_i,Y_i)-f(\tilde X_i,Y_i)$.
As in the proof of Corollary~\ref{cor:riskL-turb}, we have
\beq
\rad_n(\qloss{\paren{\H_L\oplus\turb{\eta}}};Z)
&\le&
\rad_n(\qloss{\H_L};Z) + (2q-1)\eta.
\eeq
Further,
\beq
{\rad}_n(\qloss{\H_L};Z) &=&
\E\sqprn{\sup_{f\in\qloss{\H_L}} \oo{n}\sum_{i=1}^n \sigma_i f(X_i,Y_i)
}
\\
&=&
\E\sqprn{\sup_{f\in\qloss{\H_L}} \oo{n}\sum_{i=1}^n \sigma_i \paren{f(\tilde X_i,Y_i)-\delta_i(f)}
}\\
&\le&
{\rad}_n(\qloss{\H_L};\tilde Z) 
+
\E\sqprn{\sup_{f\in\qloss{\H_L}} \oo{n}\sum_{i=1}^n \sigma_i \delta_i(f)
}.
\eeq
The first term is estimated by the same calculation as in the proof
of Theorem~\ref{thm:risk-unperturbed}:
\beq
\rad_n(\qloss{\paren{\H_L}};\tilde Z) &\le& (2q-1)\eta+\sqrt{\frac{2\lognat\calN(
(2q-1)\eta
,\qloss{\H_L},\nrm{\cdot}_\infty)}n} \\
&\le& 
\hide{
(2q-1)\eta
+\sqrt{\frac{2
\paren{\frac{16q^{3/2}{L}}{
(2q-1)\eta
}}
^{2+2\ddim(\X)}
\lognat{\frac{8}{
(2q-1)\eta
}}
}n} \\
&=&
}
(2q-1)\eta
+\sqrt{\frac2n}
\paren{\frac{16q^{3/2}{L}}{
(2q-1)\eta
}}
^{1+\beta}
\paren{\lognat{\frac{8}{
(2q-1)\eta
}}}^{1/2}.
\eeq

To bound the second term, 
invoke Lemma~\ref{lem:f2h}(i) to conclude that
\beq
\abs{\delta_i}=\abs{f(X_i,Y_i)-f(\tilde X_i,Y_i)} \le q^{3/2}L\rho(X_i,\tilde X_i).
\eeq
Hence,
\beq
\E\sqprn{\sup_{f\in\qloss{\H_L}} \oo{n}\sum_{i=1}^n \sigma_i \delta_i(f)}
&\le&
{\sup_{f\in\qloss{\H_L}} \oo{n}\sum_{i=1}^n 
\abs{f'(X_i,Y_i)-f'(\tilde X_i,Y_i)}
}\\
&\le&
n\inv q^{3/2}L\sum_{i=1}^n \rho(X_i,\tilde X_i)
\le q^{3/2}L\alpha.
\eeq
\enpf

\end{document}